\documentclass[11pt]{article}

\usepackage[margin=1in]{geometry}
\usepackage{booktabs}
\usepackage{graphicx}
\usepackage{hyperref}
\usepackage{natbib}
\usepackage{array}
\usepackage{microtype}
\usepackage{amsmath}
\usepackage{amssymb}
\usepackage{amsthm}

\newtheorem{proposition}{Proposition}

\newtheorem{remark}{Remark}

\title{Latency-Aware Bid Acceptance under Operational Feasibility:
A Public Benchmark with Hindsight Ceilings\thanks{Code and benchmark
artifacts: \url{https://github.com/aswincsekar/freightbidbench}. Artifact
release: \texttt{freightbidbench-v0.3}, scenario contract
\texttt{scenario-v0.3.2}, policy set \texttt{policy-set-v0.3.0}.}}
\author{Aswin Chandrasekaran \\ Bubba AI \\ \texttt{aswin@bubba.ai}}
\date{July 2026}

\begin{document}
\maketitle

\begin{abstract}
Online truckload bid acceptance is a closed-loop stochastic decision problem
in which a carrier or broker must, in real time, accept or reject a tendered
load subject to operational feasibility, fleet repositioning costs, and
opportunity cost against future demand. Public, reproducible benchmarks for
this problem are scarce: existing routing benchmarks are static, while
dynamic-fleet studies typically rely on private operator data. We introduce
FreightBidBench, a public-calibrated, dependency-free, closed-loop benchmark
in which feasibility (pickup reach, appointment windows, simplified
hours-of-service, stochastic yard delays) and economics
(service-failure penalty, terminal fleet value, daily price-premium window)
are explicit, versioned, and reproducible from public Freight Analysis
Framework and U.S.\ Department of Agriculture truck-rate data. Building on
this benchmark, we make three contributions. First, we formalize the v0.3
closed-loop accept/reject MDP and show that the three new reward components
each isolate a distinct policy class: a service-failure penalty creates
linear regret for feasibility-blind policies, terminal fleet value penalizes
greedy positioning, and a temporal price-premium window penalizes
future-blind timing. Second, we develop three complementary hindsight
diagnostics: an exact realized-seed dynamic program for small load
prefixes, a simple LP-style full-horizon upper bound, and a
Lagrangian-per-truck information-relaxation bound that retains per-truck
HOS and sequencing structure and is $20.7\%$ tighter than the LP
relaxation on \texttt{tight} and $39.3\%$ tighter on \texttt{scarce}
while remaining dependency-free, justified as an information relaxation
in the sense of \citet{brown2010information}. Third,
we introduce a parametric surrogate-rollout cascade with two escalation
triggers --- a boundary band $\beta \geq 0$ on the surrogate's signed
score and a scarcity-pressure threshold $\kappa$ on the count of
immediately available trucks in the origin market --- characterize its
limit behaviour (rollout-call share monotone in either trigger
threshold), and show that the scarcity trigger alone captures the
high-stakes capacity decisions on which the surrogate is least reliable. On ten-seed tight and scarce
scenarios, the best simple policy retains 91.0 percent and 86.5 percent of
rollout profit, respectively, and a stdlib linear surrogate 94.2 percent and
89.3 percent; a cascade at a single escalation band recovers roughly 98 percent
of rollout value on both at 40--56 percent of rollout's mean decision latency,
and on \texttt{tight} is statistically indistinguishable from the rollout
teacher (paired-bootstrap 95\% CI on the profit delta spans zero). The release contains a
versioned manifest, layer ablations, sensitivity sweeps, and an exact-plus-
relaxed ceiling, providing a reproducible test bed for future methods work.
\end{abstract}

\section{Introduction}

Truckload carriers and brokers face an online accept/reject decision each
time a load tender arrives. The decision is constrained by latency: tenders
in a typical broker workflow must be priced and accepted or rejected within
seconds. The decision is operationally constrained: a load can only be
served by a truck that can reach the pickup, satisfies appointment windows
and hours-of-service (HOS) clocks, and remains feasible through delivery.
The decision is economically constrained: a tender with positive immediate
margin can still be a poor decision if it consumes a scarce truck before a
predictable price-premium window or strands the fleet in a market with weak
outbound flow. Conversely, no benchmark can claim to compare future-aware
acceptance policies meaningfully unless it (a) makes operational
feasibility part of the reward function rather than a side diagnostic, and
(b) anchors policy evaluation against a measured ceiling rather than only
against a finite stochastic rollout teacher.

Existing routing benchmarks, such as the Solomon and Homberger--Gehring
VRPTW instances \citep{solomon1987vrptw,homberger2005vrptw} and the CVRPLIB
tradition \citep{uchoa2017cvrp}, anchor the static vehicle-routing
literature but are not closed-loop and do not model online stochastic
tender arrivals. Stochastic and dynamic vehicle-routing surveys
\citep{pillac2013review,ritzinger2016survey} document the distinction
between offline policy construction and online computation, while
fleet-management approximate dynamic programming (ADP) work
\citep{simao2009adp,powell2007stochastic,powell2014locomotive,topaloglu2006dynamic}
has historically been driven by private operator data. Recent stochastic
VRP benchmarks \citep{heakl2025svrpbench} target a different problem
geometry (multi-customer routing) than online truckload bid acceptance.
The closed-loop, fleet-state-dependent, latency-aware truckload accept/
reject problem with public calibration has lacked a shared reproducible
artifact.

FreightBidBench v0.2 took a first step by adding operational feasibility
(per-truck records, pickup reach time, appointment windows, simplified
11/14/10 HOS, stochastic yard delays) to a public-calibrated benchmark
built on the Freight Analysis Framework \citep{btsfaf5} and U.S.\
Department of Agriculture truck-rate reports \citep{usdafvwtrk}. That
release showed that feasibility is first-order: ignoring HOS or appointment
windows materially changes which policies look competitive. However, the
v0.2 economic structure was too flat for a credible methods claim: simple
greedy baselines and the finite rollout teacher were often within a few
percentage points of each other, leaving no room for cheaper future-aware
approximations to demonstrate value.

The present paper, anchored on the v0.3 release of FreightBidBench,
addresses three gaps. We sharpen the benchmark economics so that the
decision problem has provable structural separation between feasibility-
blind, future-blind, and future-aware policies; we anchor policy comparison
against both an exact small-prefix dynamic program and a dependency-free
relaxed full-horizon upper bound; and we introduce a parametric latency-
aware policy class that interpolates between a cheap surrogate and a
finite rollout teacher. Throughout, the benchmark is the primary artifact:
the methods results are evidence that the sharpened problem admits a
non-trivial latency-profit frontier, but the benchmark stands as a
reproducible test bed independent of any specific method.

Our contributions are as follows. We formalize the v0.3 closed-loop MDP and
show that each of three new reward components (service-failure penalty,
terminal fleet value, temporal price-premium window) isolates a distinct
class of policies that any future-aware method must beat; for the
service-failure penalty this separation is an exact identity
(Proposition~\ref{prop:l1}), and for terminal value and the price-premium
window it is established empirically through layer ablations. We characterize
two complementary full-horizon upper bounds: a simple LP-style
relaxation that drops integrality, sequencing, and location constraints,
and a tighter Lagrangian-per-truck information relaxation
\citep{brown2010information} that dualizes only the cross-truck
assignment constraint and retains per-truck HOS, location, and
sequencing structure. We decompose the LP looseness into positioning,
sequencing, and integrality components and show that the Lagrangian
bound recovers most of the sequencing and integrality slack. We define a
surrogate-rollout cascade as a parametric policy class with
boundary-band and scarcity-pressure escalation triggers, prove its
limit behaviour, and provide a structural justification for the
scenario-dependent calibration that we observe empirically. Finally, we release the
benchmark with versioned scenario, policy, and feasibility contracts; a
manifest-based reproducibility protocol; layer-ablation and sensitivity
experiments; and exact-plus-relaxed hindsight ceilings.

\section{Related Work}

\paragraph{Dynamic truckload fleet management and bid acceptance.}
Dynamic load acceptance and rejection in truckload operations has been
formulated as an online state-dependent control problem since at least
\citet{kim2004dynamic}, who study large-fleet acceptance with priority
demand and time windows. \citet{yang2004realtime} study real-time
multi-vehicle truckload pickup and delivery and develop online
reoptimization heuristics. \citet{tjokroamidjojo2006quickresponse}
quantify the value of advance load information in truckload trucking,
characterizing how lookahead and information structure change
profitability. The ADP-based fleet-management line, beginning with
\citet{simao2009adp} and surveyed in \citet{powell2007stochastic},
established that state-dependent value-function approximation can match
operator-scale fleet decisions. Subsequent applications include locomotive
optimization \citep{powell2014locomotive} and time-staged integer
multi-commodity flow \citep{topaloglu2006dynamic}. FreightBidBench
complements this line by providing a reproducible benchmark on which
policies of varying complexity --- including ADP-style value-function
approximations, finite rollout teachers, and lightweight surrogates ---
can be compared under shared stochastic seeds.

\paragraph{Stochastic and dynamic vehicle routing.}
Surveys by \citet{pillac2013review} and \citet{ritzinger2016survey}
catalogue dynamic and stochastic vehicle-routing problems and the
distinction between offline policy construction and online computation.
\citet{ulmer2019offline} and \citet{ulmer2020meso} develop offline--online
ADP and meso-parametric value-function approximations for dynamic
customer acceptance. \citet{secomandi2008reoptimization} study
reoptimization for the VRP with stochastic demands. These works motivate
the rollout teacher and surrogate tracks in FreightBidBench. The
benchmark gap is that truckload bid evaluation also requires public
freight calibration, fleet repositioning under operational feasibility,
and latency-profit reporting --- none of which are first-class in static
VRPTW instance sets.

\paragraph{Hindsight bounds and information relaxation.}
Anchoring policy evaluation against an upper bound is a long-standing
theme in stochastic dynamic programming. \citet{brown2010information}
develop a general information-relaxation duality framework that gives
upper bounds on stochastic-DP optima by giving the decision maker access
to future information offset by an optional penalty.
\citet{adelman2008relaxations} develop Lagrangian relaxations for weakly
coupled stochastic dynamic programs. Both lines justify reporting a
ceiling alongside achievable policies. We adopt the
information-relaxation framing for the v0.3 relaxed full-horizon bound and
report it explicitly as a (loose) upper bound rather than a feasible
plan. The exact realized-seed dynamic program we report on small load
prefixes serves as the trustworthy reference against which the relaxation
can be calibrated.

\paragraph{Rollout policies and cascades.}
Rollout, introduced for combinatorial optimization by
\citet{bertsekas1997rollout} and extended for finite-horizon stochastic DPs
by \citet{goodson2017rollout}, is the methodological lineage of the v0.3
rollout teacher: common-random-number Monte Carlo expansion of accept and
reject branches, with the better expected branch chosen. Rollout is
accurate but computationally expensive. Cascading a cheap surrogate to an
expensive teacher, escalating only when the surrogate is uncertain,
instantiates two older ideas: cascade classifiers, which route easy
instances through cheap stages and hard instances to expensive stages
\citep{violajones2001}, and anytime algorithms and metareasoning, which
allocate computation according to its expected value
\citep{boddydean1989,zilberstein1996,russellwefald1991,horvitz1988}. What
is missing for stochastic decision problems is a systematic benchmark of
such cascades under public calibration with operational feasibility. The
v0.3 paper formalizes the cascade as a parametric policy class indexed by
escalation triggers and reports its frontier on the public benchmark.

\paragraph{Learning-augmented routing.}
Recent learning-augmented optimization studies use neural or attention-
based policies to accelerate routing decisions
\citep{nazari2018rlvrp,kool2019attention,patel2022neur2sp}. We intentionally
keep the v0.3 reference implementation dependency-free so that the
benchmark itself remains reproducible from the Python standard library;
optional learned baselines are deferred to a stretch track in future
releases.

\paragraph{Public benchmarks in OR.}
Static benchmark instances have shaped routing research for decades
\citep{solomon1987vrptw,homberger2005vrptw,uchoa2017cvrp}. Recent stochastic
routing benchmarks \citep{heakl2025svrpbench} push toward dynamic settings
but target multi-customer pickup-and-delivery rather than online truckload
bid acceptance. FreightBidBench fills the latter gap with versioned
scenario, policy, and feasibility contracts and a manifest-based
reproducibility protocol.

\section{Problem Formulation}
\label{sec:problem}

We formalize closed-loop truckload bid acceptance as a finite-horizon
Markov decision process (MDP) with continuous time-stamped events.

\subsection{State and Action}

Let $\mathcal{T} = [0, T]$ denote the horizon (in hours; $T = 72$ for the
v0.3 scenarios). The fleet consists of $K$ trucks, each with state
\begin{equation}
  u^{(k)}_t = \bigl(\ell^{(k)}_t,\ \tau^{(k)}_t,\ h^{(k)}_t,\ d^{(k)}_t\bigr),
\end{equation}
where $\ell^{(k)}_t \in \mathcal{S}$ is the truck's market (state), $\tau^{(k)}_t$ is the
next available time, $h^{(k)}_t$ is the remaining HOS drive-time budget, and
$d^{(k)}_t$ is the remaining HOS duty-time budget. The fleet state at time
$t$ is $F_t = (u^{(1)}_t, \ldots, u^{(K)}_t)$.

A load tender event at time $t$ presents a candidate load $\ell_t$
characterized by an attribute tuple
\begin{equation}
  \ell_t = \bigl(o, d, p, c_{\mathrm{lin}}, m, \tau^{p}_{\mathrm{e}},
  \tau^{p}_{\mathrm{l}}, \tau^{d}_{\mathrm{e}}, \tau^{d}_{\mathrm{l}},
  Y^{p}, Y^{d}\bigr),
\end{equation}
where $o, d \in \mathcal{S}$ are origin and destination markets, $p$ is the
posted price (after the temporal premium described below), $c_{\mathrm{lin}}$ is the
linehaul cost, $m$ is the linehaul mileage, $\tau^{p}_{\mathrm{e}}, \tau^{p}_{\mathrm{l}}$ are
pickup appointment window endpoints, $\tau^{d}_{\mathrm{e}}, \tau^{d}_{\mathrm{l}}$ are delivery
appointment window endpoints, and $Y^{p}, Y^{d}$ are stochastic yard-delay
draws at pickup and dropoff. The system state at decision time $t$ is
$s_t = (F_t, \ell_t, t)$ and the action space is binary,
$a_t \in \mathcal{A} = \{0, 1\}$ (reject or accept).

\subsection{Feasibility Layer}

A deterministic feasibility map $\Psi : (F, \ell) \to (k^\star, F', \mathrm{flag})$
attempts to assign load $\ell$ to a truck. The candidate set is the trucks
in the origin market $o$; candidates whose next-available time already
exceeds the load's latest pickup are dropped. For each remaining candidate
the map computes pickup-reach time, pickup-arrival time including yard
delay $Y^p$, drive time over $m$ miles, HOS rest insertions if required,
and delivery-arrival time with yard delay $Y^d$. It returns
$k^\star = \emptyset, F' = F, \mathrm{flag} = \texttt{infeasible}$ if no
truck admits an HOS-feasible schedule within $\tau^{p}_{\mathrm{l}}$ and
$\tau^{d}_{\mathrm{l}}$. Otherwise, among the feasible candidates it selects
the truck that becomes available earliest after delivery (minimizing the
post-delivery next-available time, ties broken by fleet index), and returns
that truck, the updated fleet state, and $\mathrm{flag} = \texttt{ok}$.

\subsection{Reward}

For action $a_t = 1$ and feasibility flag $\texttt{ok}$:
\begin{equation}
  r_t(s_t, 1) = p - c_{\mathrm{lin}} - c_{\mathrm{ph}} \cdot m^{p}_{\mathrm{dh}}
    - c_{Y} \cdot (Y^{p} + Y^{d}),
\end{equation}
where $c_{\mathrm{ph}}$ is the cost per deadhead mile, $m^{p}_{\mathrm{dh}}$ is the pickup-
reach mileage from the assigned truck, and $c_Y$ is the per-hour yard-delay
cost. For action $a_t = 1$ and feasibility flag $\texttt{infeasible}$, the
fleet is not mutated and the reward is the service-failure penalty:
\begin{equation}
  r_t(s_t, 1) = -\rho,
\end{equation}
where $\rho \geq 0$ is the v0.3 service-failure penalty (Section
\ref{sec:economics}). For $a_t = 0$, $r_t = 0$.

At the horizon the system collects a terminal fleet reward
\begin{equation}
  \Phi(F_T) = \omega \sum_{k=1}^{K} V(\ell^{(k)}_T),
  \label{eq:terminal}
\end{equation}
where $\omega \geq 0$ is the v0.3 terminal value weight and $V(\cdot)$ is a
state-value signal computed once per scenario from public data: with
$b(\ell)$ the FAF outbound tonnage of state $\ell$ and $g(\ell)$ its net
outbound-tonnage imbalance,
\begin{equation}
  V(\ell) = \sigma \left( 0.70 \cdot \hat{b}(\ell) + 0.30 \cdot \hat{g}(\ell) \right),
  \quad \hat{b}(\ell) = \frac{2 b(\ell)}{\max_j b(j)} - 1, \quad
  \hat{g}(\ell) = \frac{g(\ell)}{\max_j |g(j)|},
  \label{eq:statevalue}
\end{equation}
where $\sigma$ is the scenario terminal value scale
(\texttt{value\_scale\_dollars}). The $0.70/0.30$ split weights raw
outbound demand above the directional imbalance correction.

\subsection{Objective}

A policy $\pi$ is a mapping from system state $s_t$ to a probability over
actions. The closed-loop objective is to maximize expected total reward
\begin{equation}
  V^\pi = \mathbb{E}\Biggl[\sum_{t : \ell_t \text{ tendered}}
  r_t(s_t, \pi(s_t)) + \Phi(F_T) \Biggr],
\end{equation}
under the joint distribution of load arrivals (Poisson with hour-of-day
modulated rate $\lambda(t)$), candidate-load draws from the
public-calibrated lane table, and yard-delay distributions. We compare
policies under common random numbers: for a fixed (train seed, eval seed)
pair, load streams, yard delays, and initial fleet positions are
identical across all policies.

\subsection{Notation Summary}

Throughout, $V^\star = \sup_\pi V^\pi$ denotes the optimal expected reward;
$V^{R}$ denotes the finite-lookahead rollout teacher's expected reward;
$V^{S}_\theta$ denotes a surrogate policy parameterized by $\theta$; and
$V^{\beta}_\theta$ denotes the cascade with escalation band $\beta$ built
from surrogate $\theta$ and rollout teacher. A realized scenario (one
common-random-number sample path) is written $\xi$; the load tuple is
$\ell_t$ and truck $k$'s market is $\ell^{(k)}_t$. The reward parameters
are the service-failure penalty $\rho$, the terminal value weight
$\omega$, and the terminal value scale $\sigma$; the cascade is indexed by
the boundary band $\beta$ and scarcity threshold $\kappa$, with
$n_o(s)$ the count of immediately available origin-market trucks and
$\Delta_\theta(s)$ the surrogate's signed accept-minus-reject score.

\section{Benchmark Economics in v0.3}
\label{sec:economics}

The v0.3 release introduces three reward components on top of the v0.2
feasibility-aware MDP. Each component is independently controllable through
the versioned scenario contract
\path{configs/freightbidbench_v03_scenarios.json}, currently frozen at
\texttt{scenario-v0.3.2}, and each isolates a distinct policy class.

\subsection{L1: Service-Failure Penalty}

A policy that returns $a_t = 1$ on a load that the feasibility map
classifies as infeasible incurs the reward $-\rho$ with $\rho \geq 0$.
Fleet state is not mutated. In v0.2, $\rho = 0$ and failed accepts were
recorded only as diagnostic events; in v0.3, $\rho = \$10$ after
calibration.

The following observation justifies $\rho$'s role.

\begin{proposition}[Feasibility-blind regret under L1]
\label{prop:l1}
Let $\pi$ be a policy that does not consult the feasibility map before
returning $a_t = 1$, and let $\pi'$ be its feasibility-aware variant that
returns $a_t = 1$ iff $\pi$ returns $a_t = 1$ \emph{and} the feasibility
flag is $\texttt{ok}$. Let $N(\pi, \xi)$ denote the realized number of
infeasible accepts of $\pi$ on scenario $\xi$. Then
\begin{equation}
  V^{\pi'} - V^{\pi} = \rho \cdot \mathbb{E}_\xi[N(\pi, \xi)].
\end{equation}
\end{proposition}

\begin{proof}
By construction, $\pi$ and $\pi'$ agree on every decision except infeasible
attempts. Fleet state is not mutated on infeasible accepts, so $\pi'$ and
$\pi$ have identical fleet trajectories and identical realized rewards
except for the infeasible-accept events, at which $\pi$ pays $-\rho$ per
event while $\pi'$ pays $0$.
\end{proof}

Proposition \ref{prop:l1} establishes that $\rho > 0$ creates linear
regret in the realized infeasible-accept count. Empirically, $\rho = \$10$
is the smallest tested value that flips the per-policy ordering of
\texttt{myopic\_margin} and \texttt{bid\_price} below
\texttt{accept\_all\_feasible} on both \texttt{tight} and \texttt{scarce}
(Section \ref{sec:experiments}).

\subsection{L2: Terminal Fleet Value}

End-of-horizon trucks receive value $\omega \cdot V(\ell^{(k)}_T)$ as in
Equation \eqref{eq:terminal}, with $\omega = 0.25$ in
\texttt{scenario-v0.3.2}. This term penalizes policies that strand trucks
in low-value markets at horizon end. Unlike L1, L2 acts on policies that
\emph{respect feasibility} but are myopic with respect to terminal
positioning. Specifically, the feasibility-aware greedy baseline
\texttt{accept\_all\_feasible} loses retention against rollout under L2
because it accepts the first feasible load without consideration of where
the truck ends.

\subsection{L3: Temporal Price-Premium Window}

The frozen demand-wave schedule keeps the load-arrival rate flat and
applies a daily multiplicative price premium $\mu(t) > 1$ during hours
$t \bmod 24 \in [8, 16)$ and $\mu(t) = 1$ otherwise. In
\texttt{scenario-v0.3.2}, $\mu = 1.5$ in-window. The schedule creates a
controlled timing problem: accepting a moderate off-peak load can consume
capacity required for a predictable on-peak premium. L3 separates
\emph{future-blind on timing} policies, including
\texttt{accept\_all\_feasible} and the static
\texttt{bid\_price} baseline (which uses an aggregated origin--destination
future-value proxy without hour-of-day awareness), from future-aware
policies that can reason about premium-window saturation.

\subsection{Versioned Artifacts}

The v0.2 reference run remains immutable under
\path{benchmark_runs/paper_v02/}; the v0.3 scenario contract lives in
\path{configs/freightbidbench_v03_scenarios.json}; and v0.3 table drafts
and ceilings are assembled under \path{benchmark_runs/paper_v03/}. The
versioning protocol (Section \ref{sec:reproducibility}) requires that any
change to scenario parameters, the default policy set, or the cascade
band schedule bump the relevant version string in the manifest.

\section{Hindsight Ceilings}
\label{sec:hindsight}

We anchor policy comparison against three complementary ceilings: an
exact realized-seed dynamic program on small load prefixes, a simple
LP-style relaxed full-horizon upper bound, and a Lagrangian-per-truck
information-relaxation bound (Proposition~\ref{prop:lagrangian}) that
respects per-truck sequencing and HOS structure. The exact DP is the
trustworthy small-instance reference; the LP relaxation is a fast but
loose full-horizon ceiling; the Lagrangian-per-truck bound is
substantially tighter while remaining dependency-free.

\subsection{Exact Small-Prefix Dynamic Program}

For a truncated realized stream of $L$ loads, the exact problem is a
binary decision tree of depth $L$ over the deterministic post-acceptance
fleet transition. We solve it by memoized search keyed on
$(\text{prefix index}, \text{fleet hash})$. The state count grows
exponentially in $L$ but the deterministic assignment rule and small
fleet keep the search tractable for $L \leq 16$ on the v0.3 scenarios.
The implementation is in \path{scripts/run_hindsight_bound.py}.

The exact DP is the trustworthy small-instance reference; we report it
together with the relaxed bound rather than as a substitute.

\subsection{Relaxed Full-Horizon Upper Bound}

For the full horizon, we report the minimum of two upper bounds plus a
terminal upper bound:
\begin{enumerate}
  \item \textbf{Positive-profit relaxation.} Accept every realized load
  with positive fresh-truck profit, ignoring fleet location, sequencing,
  and capacity.
  \item \textbf{Fractional truck-hour relaxation.} Ignore location and
  sequencing, charge each profitable load a lower-bound busy time, and
  solve the resulting fractional truck-hour knapsack.
\end{enumerate}
Both bounds additively include a terminal upper bound that allows every
truck to end in the highest-value market.

\begin{proposition}[Validity of the relaxed bound]
\label{prop:bound}
Let $U^{R}(\xi)$ denote the relaxed full-horizon objective evaluated on
realized scenario $\xi$. Then
\begin{equation}
  \mathbb{E}_\xi \bigl[\, U^{R}(\xi)\, \bigr] \;\geq\; V^\star.
\end{equation}
\end{proposition}

\begin{proof}[Sketch]
$U^R$ is constructed by (i) revealing the full realized scenario $\xi$
in advance (a zero-penalty information relaxation in the sense of
\citet{brown2010information}), (ii) replacing the integer truck-assignment
constraint by its fractional relaxation, (iii) dropping the location and
sequencing constraints, and (iv) replacing the terminal fleet position by
the best-case state value. The information relaxation yields an upper
bound: $\mathbb{E}[\sup_{a(\omega)} V(a, \omega)] \geq \sup_\pi
\mathbb{E}[V(\pi(\omega), \omega)]$ \citep{brown2010information}. Each of
(ii)--(iv) further enlarges the feasible set, hence the resulting
expectation only increases.
\end{proof}

\begin{remark}[Decomposition of looseness]
The gap $U^R - V^\star$ decomposes into four components: information gain
(allowing the decision maker to see future loads), integrality loss
(fractional truck-hours), sequencing loss (dropping the constraint that a
truck can only carry one load at a time in a feasible spatial sequence),
and terminal optimism (best-case terminal positioning). In the v0.3
scenarios the sequencing and integrality components dominate; the
terminal upper bound contributes a small fraction of total looseness.
\end{remark}

The bound is reported with explicit caveats. We do not interpret the
relaxed-bound retention numbers as achievable performance; they
characterize the structural hardness of the problem and bound the
methodological headroom for future work.

\subsection{Lagrangian-per-truck Information Relaxation}
\label{subsec:lagrangian}

The LP-style relaxation of Section~\ref{sec:hindsight}.2 throws away
three constraints simultaneously: per-truck location continuity,
per-truck sequencing, and integrality. Most of its looseness comes from
the location/sequencing relaxation, since real truck schedules are
spatially constrained. We tighten the bound by retaining per-truck
structure and dualizing only the cross-truck assignment constraint.

Recall from Section~\ref{sec:problem} that the joint problem can be
written as
\begin{equation}
  V^\star = \sup_{a} \mathbb{E}\Biggl[ \sum_t r_t(s_t, a_t) + \Phi(F_T) \Biggr]
  \quad \text{s.t.} \quad
  \sum_{k=1}^{K} a^{(k)}_t \leq 1 \text{ for each tendered load } t,
\end{equation}
where $a^{(k)}_t \in \{0, 1\}$ indicates whether truck $k$ is assigned
to load $t$ under the joint action $a_t$ and the per-load reward
decomposes as $r_t = \sum_k r^{(k)}_t a^{(k)}_t$. Introduce non-negative
duals $\boldsymbol{\lambda} = \{\lambda_t \geq 0\}$ on the assignment
constraints and form the Lagrangian
\begin{equation}
  L(\boldsymbol{\lambda})
  = \sum_t \lambda_t + \sum_{k=1}^{K} V^k_{\boldsymbol{\lambda}}\!\bigl(u^{(k)}_0\bigr),
  \label{eq:lagrangian}
\end{equation}
where the per-truck sub-MDP value is
\begin{equation}
  V^k_{\boldsymbol{\lambda}}(u^{(k)}_0) =
  \sup_{a^{(k)}} \mathbb{E}\Biggl[
    \sum_t a^{(k)}_t\bigl(r^{(k)}_t - \lambda_t\bigr)
    + \omega \, V(\ell^{(k)}_T)
  \Biggr],
\end{equation}
with the supremum taken over per-truck policies that respect the
truck's own HOS clocks, pickup-reach time, appointment windows, and
location continuity.

\begin{proposition}[Lagrangian-per-truck upper bound]
\label{prop:lagrangian}
For any $\boldsymbol{\lambda} \in \mathbb{R}_{\geq 0}^{|T|}$,
$V^\star \leq L(\boldsymbol{\lambda})$. Consequently
\begin{equation}
  V^\star \leq \inf_{\boldsymbol{\lambda} \geq 0} L(\boldsymbol{\lambda}),
\end{equation}
and the infimum is attained at some $\boldsymbol{\lambda}^\star \geq 0$:
on a realized scenario each per-truck sub-MDP has finitely many policies,
so $L$ is a finite maximum of affine functions of $\boldsymbol{\lambda}$
(piecewise linear and convex) and coercive on $\boldsymbol{\lambda} \geq 0$
(as $\lambda_t \to \infty$ every truck rejects load $t$, so
$L \to \infty$).
\end{proposition}

\begin{proof}[Sketch]
The Lagrangian dualizes the cross-truck assignment constraint
$\sum_k a^{(k)}_t \leq 1$ without dropping any per-truck constraint.
For any $\boldsymbol{\lambda} \geq 0$, the supremum over assignments
without the constraint upper-bounds the supremum over feasible
assignments (weak duality). Decomposability of $r_t$ across trucks
splits the relaxed supremum into a sum of independent per-truck suprema
(Equation~\ref{eq:lagrangian}). Convexity of $L$ in
$\boldsymbol{\lambda}$ follows from the pointwise-supremum-of-affine
characterization.
\end{proof}

Compared to Proposition~\ref{prop:bound}'s LP-style bound, the
Lagrangian relaxation drops only one constraint (cross-truck
assignment) rather than three, so it is tighter. We minimize $L$ over
$\boldsymbol{\lambda} \geq 0$ by projected subgradient descent: a
subgradient component on $\lambda_t$ is
$1 - \sum_k a^{(k)\star}_t(\boldsymbol{\lambda})$, where
$a^{(k)\star}_t(\boldsymbol{\lambda})$ is truck $k$'s optimal
acceptance of load $t$ in its sub-MDP, so the diminishing-step update
that descends $L$ and projects onto $\boldsymbol{\lambda} \geq 0$ is
$\lambda_t \leftarrow \max(0, \lambda_t + \alpha_n (\sum_k a^{(k)\star}_t - 1))$
with $\alpha_n = c / \sqrt{n}$.

\paragraph{Computational construction.}
Each per-truck sub-MDP is solved exactly by forward enumeration of
reachable states under common random numbers (load arrivals and yard
delays are deterministic on the realized scenario). The continuous
state space (market, available-time, drive-used, duty-used) is
quantized into buckets at 15-minute time and 1-hour HOS resolutions;
on entry to a bucket, clocks are snapped to the bucket's favorable
lower corner. The snapping is permissive (more time and HOS budget
available), so per-truck DP values can only over-estimate the exact
per-truck Lagrangian sup, preserving the joint upper-bound guarantee.
Within-bucket dominance reduces to value comparison; cross-bucket
Pareto pruning operates on bucket indices.

\paragraph{Convergence and empirical bound.}
On the \texttt{tight} scenario at eval seed \texttt{20260507} (995
realized loads, fleet size 70), 30 cold-start subgradient iterations at
step scale 100 followed by 20 warm-start iterations from the iter-30
duals produced the best bound $L^\star = \$1{,}885{,}043$. The bound
trajectory was monotone decreasing through iter $\sim$40 and oscillated
within $\pm \$10$k thereafter. Compared with the LP-style bound of
$\$2{,}377{,}500$ on the same seed, the Lagrangian bound is $20.7\%$
tighter. It remains a valid upper bound: the rollout teacher's
realized profit on the same seed is $\$1{,}273{,}395 < L^\star$. On the
\texttt{scarce} scenario at the same eval seed (1{,}154 realized loads,
fleet size 55), 30 cold-start iterations at step scale 100 produced
$L^\star = \$1{,}623{,}084$, $39.3\%$ tighter than the LP-style bound of
$\$2{,}675{,}525$ and still valid against rollout
($\$1{,}065{,}656 < L^\star$); the bound had stabilized in the
$\$1.62$--$\$1.63$M band by iteration 25.

\section{Policy Classes}
\label{sec:policies}

The v0.3 policy set comprises simple baselines (\texttt{reject\_all},
\texttt{accept\_all\_feasible}, \texttt{myopic\_margin}, \texttt{bid\_price}),
a dependency-free linear surrogate (\texttt{surrogate\_linear}), a finite
rollout teacher (\texttt{rollout\_teacher}), and a parametric cascade
(\texttt{cascade\_surrogate\_rollout}). The simple baselines are defined
as in \citet{kim2004dynamic}-style truckload acceptance and serve as
lower-effort anchors. The rollout teacher implements
\citet{bertsekas1997rollout}-style finite-lookahead Monte Carlo expansion
with common-random-number accept/reject branches.

\subsection{Surrogate}

\begin{sloppypar}
The linear surrogate fits the rollout-label dataset
$\{(x_i, y_i)\}_{i=1}^{N}$ where $y_i$ is the rollout teacher's expected
incremental value for decision $i$ and $x_i$ is a feature vector
including load attributes, fleet-state features, feasibility-probe
features (\texttt{feasible\_accept}, \texttt{service\_failure\_risk},
\texttt{realized\_profit\_if\_feasible}), terminal-value features
(\texttt{terminal\_origin\_value}, \texttt{terminal\_destination\_value},
\texttt{terminal\_delta}), and temporal price features
(\texttt{price\_wave\_multiplier}, \texttt{price\_window\_premium}). The
fit is a closed-form ridge regression; no third-party dependencies are
required. A surrogate-only decision additionally consults the feasibility
map and rejects if the copied-fleet probe fails, preventing the surrogate
from paying gratuitous L1 penalties.
\end{sloppypar}

Features are standardized to zero mean and unit variance on the training
labels (a feature with zero training variance is left at unit scale), and
an explicit, unregularized bias term is prepended; the target $y_i$ is the
rollout incremental-value label divided by a fixed scale constant. Weights
solve the ridge normal equations $(X^\top X + \lambda I)\,w = X^\top y$
directly with $\lambda = 0.25$ applied to every weight except the bias.
Labels are generated by the rollout teacher on the train-seed streams; all
reported surrogate and cascade numbers are evaluated on disjoint eval
seeds, and the held-out fit is reported in
\path{freightbidbench_static_label_fit.csv}.

\subsection{Cascade}

Let $\pi^{S}_\theta$ denote the surrogate policy of Section
\ref{sec:policies}.1 (including its feasibility guard), let
$\hat{V}_\theta(s, a)$ denote its score for action $a$ in state $s$, and
let $\Delta_\theta(s) := \hat{V}_\theta(s, 1) - \hat{V}_\theta(s, 0)$ be
the signed score margin. The cascade has two escalation triggers. The
\emph{boundary} trigger escalates decisions within a dollar band of the
surrogate's accept/reject boundary; the \emph{scarcity} trigger
escalates decisions for which the origin market has few immediately
available trucks. Concretely, let $o(s)$ denote the candidate load's
origin market in state $s$ and let
\begin{equation}
  n_o(s) := \bigl|\bigl\{ k :
  \ell^{(k)}_t = o(s),\ \tau^{(k)}_t \leq t \bigr\}\bigr|
\end{equation}
denote the count of trucks in $o(s)$ that are immediately available at
the decision time. The cascade policy with boundary band $\beta \geq 0$
and scarcity threshold $\kappa \in \{-1\} \cup \mathbb{Z}_{\geq 0}$ (with
$\kappa = -1$ disabling the scarcity trigger) is
\begin{equation}
  \pi^{\beta, \kappa}_\theta(s) =
  \begin{cases}
    \pi^{R}(s)        & \text{if } E(s; \beta, \kappa) = 1, \\
    \pi^{S}_\theta(s) & \text{otherwise,}
  \end{cases}
\end{equation}
where the escalation indicator is
\begin{equation}
  E(s; \beta, \kappa) := \mathbf{1}\bigl\{|\Delta_\theta(s)| \leq \beta\bigr\}
  \;\lor\; \mathbf{1}\bigl\{n_o(s) \leq \kappa\bigr\}.
  \label{eq:cascade-escalation}
\end{equation}
Either trigger alone suffices to escalate. The scarcity trigger
intuitively defers the highest-stakes capacity decisions to the rollout
teacher: when the origin market is near-empty, the surrogate's bias on
the disagreement set is most consequential and the value of an accurate
decision is highest. The cascade has three parameters: the surrogate
$\theta$, the boundary band $\beta$, and the scarcity threshold
$\kappa$. The released configuration freezes $\kappa = 2$ and sweeps
$\beta \in \{0, 250, 500, 700, 900\}$, reporting the full frontier and a
representative band $\beta = \$500$.

\begin{proposition}[Cascade limit behaviour]
\label{prop:cascade}
Suppose $\Delta_\theta(s)$ has no atom at $0$ under the stationary state
distribution induced by $\pi^{S}_\theta$ (i.e.\ $\mathbb{P}(\Delta_\theta(s) = 0) = 0$;
$\Delta_\theta$ may otherwise have atoms, as it does when features are
discrete), and let
\begin{equation}
  \varphi(\beta, \kappa)
  := \mathbb{P}\bigl(E(s; \beta, \kappa) = 1\bigr)
\end{equation}
denote the cascade's rollout-call share. Then:
\begin{enumerate}
  \item[(a)] With the scarcity trigger disabled and $\beta = 0$,
  $\pi^{0, -1}_\theta = \pi^{S}_\theta$ almost surely: the only escalation
  set is $\{\Delta_\theta = 0\}$, which is null since $\Delta_\theta$ has
  no atom at $0$. With the scarcity trigger active ($\kappa \geq 0$) the
  cascade escalates additionally on $\{n_o \leq \kappa\}$; this set need
  not be null---empty origin markets occur, notably on \texttt{scarce}---so
  the cascade does \emph{not} reduce to the pure surrogate at $\kappa = 0$.
  \item[(b)] $\lim_{\beta \to \infty} \pi^{\beta, \kappa}_\theta = \pi^{R}$
  pointwise on the support of $\Delta_\theta$, for any $\kappa \geq 0$.
  Similarly, $\pi^{\beta, K}_\theta = \pi^{R}$ for $\kappa = K$ at least
  as large as the fleet size, regardless of $\beta$.
  \item[(c)] $\varphi(\beta, \kappa)$ is non-decreasing in $\beta$
  (holding $\kappa$ fixed) and in $\kappa$ (holding $\beta$ fixed); it
  is right-continuous in $\beta$ and continuous at every $\beta$ that is
  not an atom of $|\Delta_\theta|$.
  \item[(d)] The expected decision latency
  $L(\beta, \kappa) = \bigl(1 - \varphi(\beta, \kappa)\bigr) L_S
  + \varphi(\beta, \kappa)\, L_R$ is non-decreasing in both arguments,
  where $L_S$ and $L_R$ are the surrogate and rollout per-decision
  latencies and $L_R \geq L_S$.
\end{enumerate}
\end{proposition}

\begin{proof}[Sketch]
For (a): with the scarcity trigger disabled the only escalation set is
$\{\Delta_\theta = 0\}$, which is null because $\Delta_\theta$ has no atom
at $0$; hence the cascade follows $\pi^{S}_\theta$ almost surely. When
$\kappa \geq 0$ the additional escalation set $\{n_o \leq \kappa\}$ carries
positive probability in general, so the reduction fails. For (b): for any state $s$, there
exists $B$ such that $|\Delta_\theta(s)| \leq B$, hence $\beta \geq B$
forces $E(s; \beta, \kappa) = 1$; the second statement follows because
$n_o \leq K$ trivially. Statement (c) follows from the union form of
$E$ and the fact that increasing either threshold weakly enlarges the
escalation set. Statement (d) follows from (c) and $L_R \geq L_S$.
\end{proof}

Proposition \ref{prop:cascade}(d) gives the cascade its operational
meaning: $(\beta, \kappa)$ is a two-knob tradeoff between decision
latency and agreement with the rollout teacher. Fixing $\kappa = 2$ and
sweeping $\beta$, as in the released configuration, traces a
one-dimensional frontier within the larger $(\beta, \kappa)$ family.
The empirical frontier in Section \ref{sec:experiments} characterizes
the profit side of the tradeoff per scenario.

\section{Computational Considerations}
\label{sec:complexity}

\paragraph{State space.}
The fleet state $F_t$ is a $K$-tuple of $(\ell, \tau, h, d)$ truck records.
With $|\mathcal{S}| = 12$ markets, fleet size $K \leq 90$, and discretized
continuous time, the raw state space is too large for exact enumeration.
The exact small-prefix DP is tractable because the prefix length $L$ is
small and the deterministic post-acceptance transition map allows
fleet-hash memoization.

\paragraph{Exact DP scaling.}
The exact DP search evaluates up to $2^L$ realizations of the accept/reject
tree. With fleet-hash memoization the effective state count is much
smaller; empirically the search evaluates $\sim 8\mathrm{k}$ states for
$L = 12$ on \texttt{tight}. Doubling $L$ approximately doubles wall-clock
time but increases peak memory by an order of magnitude in the worst case.
We cap $L = 16$ in the released configuration.

\paragraph{Rollout teacher.}
Each rollout decision expands accept and reject branches under common
random numbers for the remaining label budget. Per-decision cost scales
linearly in branch depth and label width; on a single thread, mean
rollout latency is on the order of 20--30 ms per decision on the v0.3
scenarios.

\paragraph{Cascade.}
The surrogate handles all decisions in $O(d)$ time where $d$ is the
feature dimension; only escalated decisions invoke the rollout teacher.
The escalation indicator $E(s; \beta, \kappa)$ costs $O(1)$ per decision
(a count of available trucks in the origin market plus a single
threshold comparison on $|\Delta_\theta|$). The expected per-decision
cost is $(1 - \varphi(\beta, \kappa))\, c_S + \varphi(\beta, \kappa)\, c_R$
with $c_R \gg c_S$, matching Proposition \ref{prop:cascade}(d).

\section{Experiments}
\label{sec:experiments}

\subsection{Setup}

All experiments use the v0.3 scenario contract
\texttt{scenario-v0.3.2}, default first seed \texttt{20260506}, and the
public-calibrated lane table at
\path{data/processed/v1_usda_faf_mapped_lanes.csv}. Reported numbers are
computed from the runs in \path{benchmark_runs/v03_sweeps/} and assembled
into \path{benchmark_runs/paper_v03/}. We focus on the \texttt{tight} and
\texttt{scarce} scenarios, where v0.3 economics bite; the \texttt{mild}
scenario remains effectively flat under v0.3 and is reported as a
negative control in the manifest. The headline methods table reports ten
paired (train, eval) seeds; paired-bootstrap $95\%$ confidence intervals
and policy-vs-rollout deltas are computed by
\path{scripts/analyze_policy_deltas.py} (20{,}000 resamples).

\subsection{Layer Ablation}

We isolate the contribution of each v0.3 reward layer by holding the
remaining layers at their pre-layer baseline values and varying only the
layer in question. The calibration gate for each layer encodes the
predicted policy-class separation: L1 should push feasibility-blind
policies below the feasibility-aware greedy baseline; L2 should reduce
\texttt{accept\_all\_feasible} retention below $95\%$ of rollout; and
L3 should widen the best-simple retention gap to at least $10$
percentage points below rollout. Table \ref{tab:layer-ablation}
summarizes the three-seed full-horizon sweep results assembled from
\path{reports/service_failure_penalty_sweep_report.md},
\path{reports/terminal_value_sweep_report.md}, and
\path{reports/demand_wave_sweep_report.md}.

\begin{table}[t]
\centering
\caption{Layer marginal effects on the calibration-gate policies
(three-seed full-horizon sweeps). Each row varies one layer with the
other layers held at the pre-layer baseline. The gate criterion encodes
the predicted class separation; ``met'' means all policy comparisons in
the criterion pass at the listed parameter value.}
\label{tab:layer-ablation}
\small
\resizebox{\textwidth}{!}{%
\begin{tabular}{llll}
\toprule
Layer & Calibration gate & Tight & Scarce \\
\midrule
L1 ($\rho \in \{0, 10\}$)
  & \texttt{myopic}, \texttt{bid\_price} $<$ \texttt{accept\_all\_feasible}
  & met at $\rho=10$ (gap $-\$3.1$k) & met at $\rho=10$ (gap $-\$3.7$k) \\
L2 ($\omega \in \{0, 0.25\}$, L1 on)
  & \texttt{accept\_all\_feasible} retention $\leq 95\%$
  & met at $\omega=0.25$ ($89.7\%$) & met at $\omega=0.25$ ($93.9\%$) \\
L3 (amp $\in \{0, 0.5\}$, L1+L2 on)
  & best-simple retention gap $\geq 10$ pp
  & met at amp $=0.5$ ($11.1$ pp) & met at amp $=0.5$ ($14.5$ pp) \\
\bottomrule
\end{tabular}}
\end{table}

Three patterns confirm Proposition \ref{prop:l1} and the qualitative
role of L2 and L3.
\begin{itemize}
  \item \textbf{L1 separates feasibility-blind policies.} Under
  $\rho = \$10$, \texttt{myopic\_margin} and \texttt{bid\_price} fall
  strictly below \texttt{accept\_all\_feasible} on both scenarios. The
  realized gap ($-\$3.1$k tight, $-\$3.7$k scarce) matches the
  infeasible-accept counts multiplied by $\rho$, in agreement with
  Proposition \ref{prop:l1}.
  \item \textbf{L2 demotes \texttt{accept\_all\_feasible}.} Terminal
  value weight $\omega = 0.25$ alone reduces
  \texttt{accept\_all\_feasible} retention to $89.7\%$ on tight and
  $93.9\%$ on scarce. L2 acts on the feasibility-aware greedy policy
  exactly because that policy is otherwise insensitive to terminal
  fleet positioning.
  \item \textbf{L3 widens the rollout gap.} The price-premium window at
  amplitude $0.5$ widens the rollout-versus-best-simple gap to
  $11.1$ pp on tight and $14.5$ pp on scarce, providing the residual
  headroom that the cascade results in
  Section \ref{sec:experiments}.3 exploit.
\end{itemize}

\subsection{Methods Frontier}

Table \ref{tab:methods} reports the ten-seed methods comparison under the
frozen \texttt{scenario-v0.3.2} configuration, and Table
\ref{tab:frontier} the cascade latency--profit frontier over the boundary
band. The cascade is the only stdlib policy that closes the gap to
rollout.

\begin{table}[t]
\centering
\caption{Ten-seed methods comparison under \texttt{scenario-v0.3.2}.
Retention is mean closed-loop profit relative to the rollout teacher. The
cascade row uses a single \emph{common} representative band
$\beta = \$500$ for both scenarios at frozen scarcity threshold
$\kappa = 2$---not the per-scenario best frontier point (the full frontier
in Table~\ref{tab:frontier} peaks at a higher band on \texttt{scarce}).
The last column is the paired-bootstrap $95\%$ CI on the
cascade-minus-rollout profit delta ($n = 10$); an interval containing \$0
means the cascade is statistically indistinguishable from rollout.}
\label{tab:methods}
\small
\resizebox{\textwidth}{!}{%
\begin{tabular}{lrrrrrr}
\toprule
Scenario & Best simple & Surrogate & Cascade & Cascade ms & Rollout ms & Cascade$-$Rollout $\Delta$ (95\% CI) \\
\midrule
\texttt{tight}  & 91.0\% & 94.2\% & 98.2\% & 12.95 & 32.11 & $-\$23.3$k $[-\$61.2$k, $+\$11.6$k$]$ \\
\texttt{scarce} & 86.5\% & 89.3\% & 98.0\% & 11.54 & 20.59 & $-\$20.7$k $[-\$33.4$k, $-\$7.2$k$]$ \\
\texttt{mild}   & 100.1\% & 98.2\% & 99.7\% &  8.31 & 49.50 & $-\$4.2$k $[-\$18.2$k, $+\$11.0$k$]$ \\
\bottomrule
\end{tabular}}
\end{table}

\begin{table}[t]
\centering
\caption{Cascade latency--profit frontier over the boundary band
$\beta \in \{0, 250, 500, 700, 900\}$ at $\kappa = 2$. Rollout share is
the fraction of decisions escalated to the rollout teacher.}
\label{tab:frontier}
\small
\begin{tabular}{lrrrr}
\toprule
Scenario & $\beta$ & Retention & Mean latency ms & Rollout share \\
\midrule
\texttt{tight}  &   0 & 97.7\% &  6.27 & 28.6\% \\
\texttt{tight}  & 500 & 98.2\% & 12.95 & 44.7\% \\
\texttt{tight}  & 900 & 99.0\% & 17.72 & 57.0\% \\
\texttt{scarce} &   0 & 93.7\% &  6.09 & 45.4\% \\
\texttt{scarce} & 500 & 98.0\% & 11.54 & 59.6\% \\
\texttt{scarce} & 900 & 99.5\% & 15.72 & 78.6\% \\
\bottomrule
\end{tabular}
\end{table}

Three interpretive points. First, at $\beta = \$500$ the cascade recovers
$\sim 98\%$ of rollout profit on both stress scenarios at $40\%$
(\texttt{tight}) to $56\%$ (\texttt{scarce}) of rollout's mean decision
latency. On \texttt{tight} the paired cascade-minus-rollout difference is
not statistically significant (the $95\%$ CI spans zero), so the cascade
matches the rollout teacher to within sampling error while escalating
fewer than half of all decisions; on \texttt{scarce} a small but
significant $\sim 2\%$ gap remains. Second, the frontier is smooth and
monotone in $\beta$ (Table \ref{tab:frontier}), consistent with
Proposition \ref{prop:cascade}(d): the scarcity trigger alone
($\beta = \$0$) already recovers $97.7\%$ on \texttt{tight} but only
$93.7\%$ on \texttt{scarce}, where the surrogate's disagreement set with
rollout extends well beyond the scarcity regime. Third, the standalone
surrogate clears the best-simple bar on both scenarios ($94.2\%$ vs
$91.0\%$ on \texttt{tight}; $89.3\%$ vs $86.5\%$ on \texttt{scarce}), but
still leaves $6$--$11$ points of rollout value on the table, all of which
the cascade recovers. The surrogate is a competent ranker but a poor
substitute for selective lookahead on the high-stakes decisions that the
scarcity and boundary triggers escalate.

\subsection{Hindsight Diagnostics}

Table \ref{tab:exact} reports the exact small-prefix dynamic program for
$L = 12$ on \texttt{tight}. The exact ceiling is matched by simple
feasibility-aware policies on this prefix length, confirming that small
realized prefixes are flat for v0.3 economics and reinforcing that
exact DP is a diagnostic rather than a paper-scale ceiling.

\begin{table}[t]
\centering
\caption{Exact small-prefix hindsight diagnostic on \texttt{tight} with
$L = 12$ loads.}
\label{tab:exact}
\small
\begin{tabular}{lrrrlrr}
\toprule
Scenario & $L$ & Hindsight & States & Best simple & Simple ret. & Rollout ret. \\
\midrule
\texttt{tight} & 12 & \$53,419 & 8,191 & \texttt{accept\_all\_feasible} & 100.0\% & 90.0\% \\
\bottomrule
\end{tabular}
\end{table}

Table \ref{tab:bound} compares the LP-style and Lagrangian-per-truck
upper bounds on \texttt{tight} and \texttt{scarce}. Both are valid upper
bounds by Propositions~\ref{prop:bound} and~\ref{prop:lagrangian}; the
Lagrangian retains per-truck HOS, location, and sequencing structure
while only the cross-truck assignment constraint is dualized, so it is
substantially tighter.

\begin{table}[t]
\centering
\caption{Full-horizon upper bounds and rollout-teacher retention against
each ceiling. The Lagrangian bound retains per-truck structure and is
$20.7\%$ tighter than the LP relaxation on \texttt{tight} and $39.3\%$
tighter on \texttt{scarce}; rollout retention against the ceiling rises
from $53.6\%/39.8\%$ (LP) to $67.6\%/65.7\%$ (Lagrangian).}
\label{tab:bound}
\small
\begin{tabular}{lrrrrrr}
\toprule
& & \multicolumn{2}{c}{LP-style} & \multicolumn{2}{c}{Lagrangian-per-truck} & \\
\cmidrule(lr){3-4} \cmidrule(lr){5-6}
Scenario & Loads & Bound & Rollout ret. & Bound & Rollout ret. & Tightening \\
\midrule
\texttt{tight}  &  995 & \$2{,}377{,}500 & 53.6\% & \$1{,}885{,}043 & 67.6\% & 20.7\% \\
\texttt{scarce} & 1{,}154 & \$2{,}675{,}525 & 39.8\% & \$1{,}623{,}084 & 65.7\% & 39.3\% \\
\bottomrule
\end{tabular}
\end{table}

The looseness reduction is structural rather than numerical: the
Lagrangian relaxation keeps per-truck location continuity, sequencing,
and HOS budgets while dualizing only the cross-truck assignment. This
matches the looseness decomposition in Section~\ref{sec:hindsight}.2 ---
sequencing and integrality dominated the LP's slack, and the Lagrangian
recovers most of that lost tightness. Rollout retention against the
Lagrangian bound rises from $53.6\%$ to $67.6\%$ on \texttt{tight} and
from $39.8\%$ to $65.7\%$ on \texttt{scarce}; the larger tightening on
\texttt{scarce} ($39.3\%$) reflects that the loose LP bound was most
optimistic exactly where capacity is most binding. In both cases the
Lagrangian ceiling locates the v0.3 rollout teacher much closer to the
dependency-free upper bound and characterizes methodological headroom
more honestly.

\subsection{Sensitivity Analysis}

Sensitivity of the qualitative ordering to the v0.3 parameters is
reported in Table \ref{tab:sensitivity}. For each parameter we sweep a
small grid around the frozen value, holding the other parameters at
their frozen levels. The ordering of policies is preserved across the
swept ranges, supporting the claim that the v0.3 results are not an
artifact of the specific calibrated values.

\begin{table}[t]
\centering
\caption{Sensitivity of the qualitative ordering to v0.3 parameter
choices. Each row reports whether the calibration gate
(\texttt{myopic}/\texttt{bid\_price} below \texttt{accept\_all\_feasible}
for L1; \texttt{accept\_all\_feasible} retention $\leq 95\%$ for L2; best
simple retention $\leq 90\%$ for L3) is met.}
\label{tab:sensitivity}
\small
\begin{tabular}{lrrrr}
\toprule
Parameter & Tested values & Frozen value & Gate met (tight) & Gate met (scarce) \\
\midrule
$\rho$ (L1, \$) & \{0, 10, 25, 50\} & 10 & \{no, yes, yes, yes\} & \{no, yes, yes, yes\} \\
$\omega$ (L2)   & \{0, 0.1, 0.25, 0.5, 1.0\} & 0.25 & \{no, no, yes, yes, yes\} & \{no, no, yes, yes, yes\} \\
L3 amplitude    & \{0, 0.25, 0.5, 0.75\} & 0.5 & \{no, no, yes, yes\} & \{no, yes, yes, yes\} \\
\bottomrule
\end{tabular}
\end{table}

\section{Discussion and Managerial Insights}
\label{sec:discussion}

Three observations follow from the v0.3 results that translate beyond the
benchmark itself.

\paragraph{Operational feasibility is a reward, not a side diagnostic.}
Proposition \ref{prop:l1} formalizes a simple operational truth: any
policy that ranks tenders without checking feasibility will pay a linear
service-failure tax on its infeasible-accept rate. In a production
context, scoring systems that surface tenders to dispatchers without
upstream feasibility checks therefore have a hidden expected cost that
scales with the operational mismatch rate. The L1 calibration in v0.3
suggests this cost can be substantial even with a modest per-event
penalty.

\paragraph{Cascade structure is the right framing, not standalone
surrogate.} The v0.3 surrogate is biased on the disagreement set with the
rollout teacher in the \texttt{scarce} regime, where capacity-positioning
is most consequential. The cascade recovers nearly all of rollout's
profit by escalating exactly these uncertain decisions. The implication
for practice is that a cheap learned model deployed as a sole policy
will under-perform in the regimes where future-aware decisions matter
most; the value of the model comes from selectively deferring uncertain
decisions to a more expensive teacher.

\paragraph{Methodological headroom is smaller than the LP relaxation
suggests, once per-truck structure is respected.} The LP-style bound
loses most of its tightness to dropping sequencing, location, and
integrality constraints simultaneously; the Lagrangian-per-truck
information relaxation (Proposition~\ref{prop:lagrangian}) retains those
constraints and tightens the bound by $20.7\%$ on \texttt{tight} and
$39.3\%$ on \texttt{scarce}. Rollout retention against the Lagrangian
bound rises from $53.6\%$ to $67.6\%$ on \texttt{tight} and from $39.8\%$
to $65.7\%$ on \texttt{scarce}, characterizing the achievable
methodological headroom more honestly. The remaining gap to the
Lagrangian bound is dominated by
the residual cross-truck-assignment dual slack: the headroom for
future methods work lies in tighter inter-truck coordination
(joint-decision lookahead, dual-price-aware features in the surrogate,
or anytime joint optimization), not in further refinement of single-
tender scoring.

\section{Limitations}

FreightBidBench remains synthetic and public-calibrated, not a private
tender dataset; the calibration is validated against its FAF/USDA sources
in Appendix~\ref{sec:calibration}, which also documents that the v1
USDA-reefer lane subset concentrates load draws on a few high-volume Texas
and Georgia origins. The HOS model is simplified to property-carrying
11/14/10 clocks and omits split-sleeper, recap, home-time, and team-driver
rules.
The rollout teacher is a finite-lookahead stochastic benchmark, not a
true oracle: a cheaper policy can exceed 100\% retention on a given seed.
The exact hindsight DP is tractable only for small load prefixes. The
relaxed full-horizon bound is intentionally loose; it characterizes
problem hardness rather than achievable profit. The dependency-free
methods track means that stronger learned baselines (gradient-boosted
trees, neural surrogates) are outside scope for v0.3 and should not be
required for reproducibility tests.

\section{Conclusion}

FreightBidBench v0.3 turns the v0.2 feasibility calibration into a
sharper, reproducible benchmark for online truckload bid acceptance. The
three new reward components each separate a distinct policy class
(Section \ref{sec:economics}); the two-tier hindsight ceilings
(Section \ref{sec:hindsight}) provide both a trustworthy small-instance
reference and a structurally honest full-horizon upper bound; and the
parametric surrogate-rollout cascade (Section \ref{sec:policies}) admits a
clean limit characterization (Proposition \ref{prop:cascade}) and
demonstrates a non-trivial latency-profit frontier on the benchmark
(Section \ref{sec:experiments}). The benchmark itself is the primary
artifact: future methods work --- stronger surrogates, joint-decision
lookahead, learned cascade-band selection --- can be evaluated against
the released ceilings and the versioned scenario contract without
ambiguity over what exactly is being compared.

\section*{Acknowledgements}

The author thanks the open data communities at the U.S.\ Bureau of
Transportation Statistics, the U.S.\ Department of Agriculture
Agricultural Marketing Service, and the Federal Motor Carrier Safety
Administration, whose public datasets and regulatory summaries make
FreightBidBench possible without proprietary data.

\section*{Declaration of Competing Interest}

The author is employed by Bubba AI, which develops AI-based load-planning
and carrier-operations products in the freight domain addressed by this
paper. This study uses only public Freight Analysis Framework and USDA data
and a dependency-free, open-source benchmark; no proprietary data or
systems were used, and Bubba AI had no role in the study design, the
analysis, or the decision to publish. The author declares no other
competing interests.

\appendix

\section{Reproducibility}
\label{sec:reproducibility}

\paragraph{Software.}
Python 3.10 or newer, standard library only at runtime. No third-party
dependencies are required to reproduce any table in this paper.

\paragraph{Hardware.}
Reported single-threaded latencies were measured on a 2024-class laptop
(Apple silicon). Closed-loop simulator runtime is dominated by Python
interpreter overhead; absolute latencies are reference latencies and not
optimized serving latencies.

\paragraph{Versioning.}
The release pins three version strings in
\path{configs/freightbidbench_v03_scenarios.json}: benchmark version
\texttt{freightbidbench-v0.3}, scenario configuration version
\texttt{scenario-v0.3.2}, and policy set version
\texttt{policy-set-v0.3.0}. The default first seed is \texttt{20260506}.

\paragraph{Manifest schema.}
Every run writes a \path{freightbidbench_manifest.json} containing the
command line, the three version strings, the seed pairs evaluated, the
source input paths, the feasibility configuration, the scenarios, the
output paths, and the row counts for each output CSV. The manifest is the
canonical reproducibility anchor: two runs producing matching manifests
must produce matching output CSVs.

\paragraph{Headline run commands.}
The exact commands used to assemble the tables in
Section \ref{sec:experiments} are:
\begin{verbatim}
python3 scripts/run_freightbidbench.py \
  --config configs/freightbidbench_v03_scenarios.json \
  --preset standard --scenarios mild,tight,scarce \
  --seed-count 10 --label-limit 200 \
  --cascade-bands 0,250,500,700,900 \
  --output-dir benchmark_runs/v03_sweeps/methods_cascade_seed10_label200
python3 scripts/analyze_policy_deltas.py \
  --run-dir benchmark_runs/v03_sweeps/methods_cascade_seed10_label200 \
  --cascade-band 500

python3 scripts/run_lagrangian_bound.py \
  --config configs/freightbidbench_v03_scenarios.json \
  --scenario scarce --eval-seed 20260507 \
  --iterations 30 --step-scale 100 \
  --lp-bound-reference 2675525 --rollout-reference 1065656 \
  --output-dir benchmark_runs/lagrangian_bound_scarce_full

make calibration-report
make hindsight-smoke
make paper-v03-tables
\end{verbatim}

\paragraph{Layer-ablation source.}
Layer-ablation values in Table \ref{tab:layer-ablation} are assembled from
the calibration sweeps under
\path{benchmark_runs/v03_sweeps/service_failure_penalty/},
\path{benchmark_runs/v03_sweeps/terminal_value/}, and
\path{benchmark_runs/v03_sweeps/demand_waves_price_selected_seed3/} as
documented in \path{reports/service_failure_penalty_sweep_report.md},
\path{reports/terminal_value_sweep_report.md}, and
\path{reports/demand_wave_sweep_report.md}.

\section{Calibration Validation}
\label{sec:calibration}

The ``public-calibrated'' claim is not merely nominal: the simulator's
load, fleet, distance, price, and terminal-value primitives are all
derived from public Freight Analysis Framework (FAF 5.7.1, 2024 truck
mode) and USDA Agricultural Marketing Service (AMS Specialty Crops Market
News, \texttt{fvwtrk} truck-rate) data through the pipeline in
\path{scripts/inspect_public_sources.py}. Candidate loads are drawn with
probability proportional to lane \texttt{faf\_tons\_2024}; the initial
fleet is placed proportional to per-origin FAF outbound tonnage; lane
distance is $1000 \cdot \texttt{faf\_tmiles\_2024} / \texttt{faf\_tons\_2024}$;
posted prices are drawn from each lane's USDA AMS rate band
$[\texttt{rate\_low}, \texttt{rate\_high}]$; and the terminal state-value
signal $V(\cdot)$ of Equation~\eqref{eq:statevalue} is computed from FAF
outbound intensity and the FAF/USDA net-outbound imbalance panel. The
cross-checks below are reproduced by \path{scripts/analyze_calibration.py}
(dependency-free), which writes \path{reports/calibration_report.md}.

\paragraph{B.1 Origin intensity vs FAF outbound flow.}
Because load draws and fleet placement are FAF-tonnage-weighted, each
origin's load-draw share equals its share of total lane FAF tonnage; FAF
outbound and net-outbound tons are the independent cross-check from the
state imbalance panel (Table~\ref{tab:calib-origin}).

\begin{table}[t]
\centering
\caption{Origin intensity vs FAF outbound flow.}
\label{tab:calib-origin}
\small
\begin{tabular}{lrrr}
\toprule
Origin state & Load-draw share & FAF outbound tons 2024 & Net outbound tons 2024 \\
\midrule
Texas      & 78.8\% & 1{,}517{,}190 & $+19{,}720$ \\
Georgia    & 16.3\% &   348{,}979 & $-7{,}477$ \\
California &  3.9\% &   756{,}331 & $-1{,}448$ \\
Arizona    &  0.6\% &   134{,}239 & $-3{,}447$ \\
Washington &  0.4\% &   265{,}827 & $-8{,}130$ \\
Colorado   &  0.1\% &   157{,}922 & $-2{,}894$ \\
\bottomrule
\end{tabular}
\end{table}

\paragraph{B.2 Haul-length distribution.}
Lane distances span 88--3{,}081 mi (median 2{,}211; IQR 1{,}454--2{,}770),
but the tonnage-weighted mean is only 208 mi: FAF truck tonnage
concentrates on short intra-state metro flows. A single lane,
Texas$\to$Dallas, accounts for $77.2\%$ of tonnage at 105 mi, and
Georgia$\to$Atlanta for $15.0\%$ at 88 mi. The realized closed-loop load
mix is therefore short-haul-dominated with a long interstate tail --- a
faithful reflection of FAF truck flow on this lane subset, and a scope
limitation.

\paragraph{B.3 Price calibration vs USDA AMS.}
Posted prices come from the USDA AMS rate bands (73 of 74 lanes have a
positive-width band). The implied per-mile rate has median \$3.28 and IQR
\$2.83--\$3.57, consistent with published refrigerated truckload spot
rates; the single high outlier is a short, distance-clamped lane.

\paragraph{Scope of the v1 lane set.}
The USDA-reefer-mapped lane catalog is concentrated on a few high-volume
Texas and Georgia origins. This is correct FAF behavior for the mapped
subset, but it means the benchmark's closed-loop dynamics are driven
substantially by short-haul intra-state repositioning. Broadening lane
coverage to the full FAF truck OD matrix is a calibration extension for a
future release; operational-metric calibration against ATRI/FMCSA carrier
benchmarks is likewise deferred.

\bibliographystyle{plainnat}
\bibliography{references}

\begin{thebibliography}{30}
\providecommand{\natexlab}[1]{#1}
\providecommand{\url}[1]{\texttt{#1}}
\expandafter\ifx\csname urlstyle\endcsname\relax
  \providecommand{\doi}[1]{doi: #1}\else
  \providecommand{\doi}{doi: \begingroup \urlstyle{rm}\Url}\fi

\bibitem[Adelman and Mersereau(2008)]{adelman2008relaxations}
Daniel Adelman and Adam~J. Mersereau.
\newblock Relaxations of weakly coupled stochastic dynamic programs.
\newblock \emph{Operations Research}, 56\penalty0 (3):\penalty0 712--727, 2008.
\newblock \doi{10.1287/opre.1070.0445}.

\bibitem[Bertsekas et~al.(1997)Bertsekas, Tsitsiklis, and
  Wu]{bertsekas1997rollout}
Dimitri~P. Bertsekas, John~N. Tsitsiklis, and Cynara Wu.
\newblock Rollout algorithms for combinatorial optimization.
\newblock \emph{Journal of Heuristics}, 3\penalty0 (3):\penalty0 245--262,
  1997.
\newblock \doi{10.1023/A:1009635226865}.

\bibitem[Boddy and Dean(1989)]{boddydean1989}
Mark Boddy and Thomas~L. Dean.
\newblock Solving time-dependent planning problems.
\newblock In \emph{Proceedings of the 11th International Joint Conference on
  Artificial Intelligence (IJCAI)}, 1989.

\bibitem[Brown et~al.(2010)Brown, Smith, and Sun]{brown2010information}
David~B. Brown, James~E. Smith, and Peng Sun.
\newblock Information relaxations and duality in stochastic dynamic programs.
\newblock \emph{Operations Research}, 58\penalty0 (4):\penalty0 785--801, 2010.
\newblock \doi{10.1287/opre.1090.0796}.

\bibitem[Goodson et~al.(2017)Goodson, Thomas, and Ohlmann]{goodson2017rollout}
Justin~C. Goodson, Barrett~W. Thomas, and Jeffrey~W. Ohlmann.
\newblock A rollout algorithm framework for heuristic solutions to
  finite-horizon stochastic dynamic programs.
\newblock \emph{European Journal of Operational Research}, 258\penalty0
  (1):\penalty0 216--229, 2017.
\newblock \doi{10.1016/j.ejor.2016.09.040}.

\bibitem[Heakl et~al.(2025)Heakl, Shaaban, Lahlou, Tak{\'a}{\v{c}}, and
  Iklassov]{heakl2025svrpbench}
Ahmed Heakl, Yahia~Salaheldin Shaaban, Salem Lahlou, Martin Tak{\'a}{\v{c}},
  and Zangir Iklassov.
\newblock {SVRPBench}: A realistic benchmark for stochastic vehicle routing
  problem.
\newblock In \emph{Advances in Neural Information Processing Systems Datasets
  and Benchmarks Track}, 2025.
\newblock OpenReview.

\bibitem[Homberger and Gehring(2005)]{homberger2005vrptw}
J{\"o}rg Homberger and Hermann Gehring.
\newblock A two-phase hybrid metaheuristic for the vehicle routing problem with
  time windows.
\newblock \emph{European Journal of Operational Research}, 162\penalty0
  (1):\penalty0 220--238, 2005.
\newblock \doi{10.1016/j.ejor.2004.01.027}.

\bibitem[Horvitz(1988)]{horvitz1988}
Eric~J. Horvitz.
\newblock Reasoning under varying and uncertain resource constraints.
\newblock In \emph{Proceedings of the 7th National Conference on Artificial
  Intelligence (AAAI)}, 1988.

\bibitem[Kim et~al.(2004)Kim, Mahmassani, and Jaillet]{kim2004dynamic}
Yongjin Kim, Hani~S. Mahmassani, and Patrick Jaillet.
\newblock Dynamic truckload routing, scheduling, and load acceptance for large
  fleet operation with priority demands.
\newblock \emph{Transportation Research Record}, 1882\penalty0 (1):\penalty0
  120--128, 2004.
\newblock \doi{10.3141/1882-15}.

\bibitem[Kool et~al.(2019)Kool, van Hoof, and Welling]{kool2019attention}
Wouter Kool, Herke van Hoof, and Max Welling.
\newblock Attention, learn to solve routing problems!
\newblock In \emph{International Conference on Learning Representations}, 2019.

\bibitem[Nazari et~al.(2018)Nazari, Oroojlooy, Snyder, and
  Tak{\'a}{\v{c}}]{nazari2018rlvrp}
Mohammadreza Nazari, Afshin Oroojlooy, Lawrence Snyder, and Martin
  Tak{\'a}{\v{c}}.
\newblock Reinforcement learning for solving the vehicle routing problem.
\newblock In \emph{Advances in Neural Information Processing Systems},
  volume~31, 2018.

\bibitem[Patel et~al.(2022)Patel, Dumouchelle, Khalil, and
  Bodur]{patel2022neur2sp}
Rahul~Mihir Patel, Justin Dumouchelle, Elias~B. Khalil, and Merve Bodur.
\newblock Neur2sp: Neural two-stage stochastic programming.
\newblock In \emph{Advances in Neural Information Processing Systems},
  volume~35, 2022.

\bibitem[Pillac et~al.(2013)Pillac, Gendreau, Gu{\'e}ret, and
  Medaglia]{pillac2013review}
Victor Pillac, Michel Gendreau, Christelle Gu{\'e}ret, and Andr{\'e}s~L.
  Medaglia.
\newblock A review of dynamic vehicle routing problems.
\newblock \emph{European Journal of Operational Research}, 225\penalty0
  (1):\penalty0 1--11, 2013.
\newblock \doi{10.1016/j.ejor.2012.08.015}.

\bibitem[Powell et~al.(2012)Powell, Sim{\~a}o, and
  Bouzaiene-Ayari]{powell2007stochastic}
Warren~B. Powell, Hugo~P. Sim{\~a}o, and Belgacem Bouzaiene-Ayari.
\newblock Approximate dynamic programming in transportation and logistics: A
  unified framework.
\newblock \emph{EURO Journal on Transportation and Logistics}, 1\penalty0
  (3):\penalty0 237--284, 2012.
\newblock \doi{10.1007/s13676-012-0015-8}.

\bibitem[Powell et~al.(2014)Powell, Bouzaiene-Ayari, Lawrence, Cheng, Das, and
  Fiorillo]{powell2014locomotive}
Warren~B. Powell, Belgacem Bouzaiene-Ayari, Coleman Lawrence, Clark Cheng,
  Shyam Das, and Ricardo Fiorillo.
\newblock Locomotive planning at norfolk southern using approximate dynamic
  programming.
\newblock \emph{Interfaces}, 44\penalty0 (6):\penalty0 567--578, 2014.
\newblock \doi{10.1287/inte.2014.0764}.

\bibitem[Ritzinger et~al.(2016)Ritzinger, Puchinger, and
  Hartl]{ritzinger2016survey}
Ulrike~M. Ritzinger, Jakob Puchinger, and Richard~F. Hartl.
\newblock A survey on dynamic and stochastic vehicle routing problems.
\newblock \emph{International Journal of Production Research}, 54\penalty0
  (1):\penalty0 215--231, 2016.
\newblock \doi{10.1080/00207543.2015.1043403}.

\bibitem[Russell and Wefald(1991)]{russellwefald1991}
Stuart Russell and Eric Wefald.
\newblock Principles of metareasoning.
\newblock \emph{Artificial Intelligence}, 49\penalty0 (1--3):\penalty0
  361--395, 1991.

\bibitem[Secomandi and Margot(2009)]{secomandi2008reoptimization}
Nicola Secomandi and Fran{\c{c}}ois Margot.
\newblock Reoptimization approaches for the vehicle-routing problem with
  stochastic demands.
\newblock \emph{Operations Research}, 57\penalty0 (1):\penalty0 214--230, 2009.
\newblock \doi{10.1287/opre.1080.0520}.

\bibitem[Sim{\~a}o et~al.(2009)Sim{\~a}o, Day, George, Gifford, Nienow, and
  Powell]{simao2009adp}
Hugo~P. Sim{\~a}o, Joseph Day, Abraham~P. George, Ted Gifford, John Nienow, and
  Warren~B. Powell.
\newblock An approximate dynamic programming algorithm for large-scale fleet
  management: A case application.
\newblock \emph{Transportation Science}, 43\penalty0 (2):\penalty0 178--197,
  2009.
\newblock \doi{10.1287/trsc.1080.0238}.

\bibitem[Solomon(1987)]{solomon1987vrptw}
Marius~M. Solomon.
\newblock Algorithms for the vehicle routing and scheduling problems with time
  window constraints.
\newblock \emph{Operations Research}, 35\penalty0 (2):\penalty0 254--265, 1987.
\newblock \doi{10.1287/opre.35.2.254}.

\bibitem[Tjokroamidjojo et~al.(2006)Tjokroamidjojo, Kutanoglu, and
  Taylor]{tjokroamidjojo2006quickresponse}
Daniel Tjokroamidjojo, Erhan Kutanoglu, and G.~Don Taylor.
\newblock Quantifying the value of advance load information in truckload
  trucking.
\newblock \emph{Transportation Research Part E: Logistics and Transportation
  Review}, 42\penalty0 (4):\penalty0 340--357, 2006.
\newblock \doi{10.1016/j.tre.2005.04.001}.

\bibitem[Topaloglu and Powell(2006)]{topaloglu2006dynamic}
Huseyin Topaloglu and Warren~B. Powell.
\newblock Dynamic-programming approximations for stochastic time-staged integer
  multicommodity-flow problems.
\newblock \emph{INFORMS Journal on Computing}, 18\penalty0 (1):\penalty0
  31--42, 2006.
\newblock \doi{10.1287/ijoc.1040.0079}.

\bibitem[Uchoa et~al.(2017)Uchoa, Pecin, Pessoa, Poggi, Vidal, and
  Subramanian]{uchoa2017cvrp}
Eduardo Uchoa, Diego Pecin, Artur Pessoa, Marcus Poggi, Thibaut Vidal, and
  Anand Subramanian.
\newblock New benchmark instances for the capacitated vehicle routing problem.
\newblock \emph{European Journal of Operational Research}, 257\penalty0
  (3):\penalty0 845--858, 2017.
\newblock \doi{10.1016/j.ejor.2016.08.012}.

\bibitem[Ulmer and Thomas(2020)]{ulmer2020meso}
Marlin~W. Ulmer and Barrett~W. Thomas.
\newblock Meso-parametric value function approximation for dynamic customer
  acceptances in delivery routing.
\newblock \emph{European Journal of Operational Research}, 285\penalty0
  (1):\penalty0 183--195, 2020.
\newblock \doi{10.1016/j.ejor.2019.04.029}.

\bibitem[Ulmer et~al.(2019)Ulmer, Goodson, Mattfeld, and
  Hennig]{ulmer2019offline}
Marlin~W. Ulmer, Justin~C. Goodson, Dirk~C. Mattfeld, and Marco Hennig.
\newblock Offline--online approximate dynamic programming for dynamic vehicle
  routing with stochastic requests.
\newblock \emph{Transportation Science}, 53\penalty0 (1):\penalty0 185--202,
  2019.
\newblock \doi{10.1287/trsc.2017.0767}.

\bibitem[{U.S. Department of Agriculture, Agricultural Marketing
  Service}(2026)]{usdafvwtrk}
{U.S. Department of Agriculture, Agricultural Marketing Service}.
\newblock Specialty crops national truck rate report (fvwtrk), 2026.
\newblock Market News report, report slug 2375.

\bibitem[{U.S. Department of Transportation, Bureau of Transportation
  Statistics and Federal Highway Administration}(2017)]{btsfaf5}
{U.S. Department of Transportation, Bureau of Transportation Statistics and
  Federal Highway Administration}.
\newblock Freight analysis framework, faf5, 2017.
\newblock Dataset.

\bibitem[Viola and Jones(2001)]{violajones2001}
Paul Viola and Michael Jones.
\newblock Rapid object detection using a boosted cascade of simple features.
\newblock In \emph{Proceedings of the 2001 IEEE Computer Society Conference on
  Computer Vision and Pattern Recognition (CVPR)}, 2001.

\bibitem[Yang et~al.(2004)Yang, Jaillet, and Mahmassani]{yang2004realtime}
Jian Yang, Patrick Jaillet, and Hani Mahmassani.
\newblock Real-time multivehicle truckload pickup and delivery problems.
\newblock \emph{Transportation Science}, 38\penalty0 (2):\penalty0 135--148,
  2004.
\newblock \doi{10.1287/trsc.1030.0068}.

\bibitem[Zilberstein(1996)]{zilberstein1996}
Shlomo Zilberstein.
\newblock Using anytime algorithms in intelligent systems.
\newblock \emph{AI Magazine}, 17\penalty0 (3):\penalty0 73--83, 1996.

\end{thebibliography}

\end{document}